%% file: paper.tex
\pdfoutput=1
\documentclass[letterpaper, 10 pt, conference]{ieeeconf}  % Comment this line out if you need a4paper

\IEEEoverridecommandlockouts                              % This command is only needed if 
\usepackage{mathrsfs}
\usepackage[font={small}]{caption}
\usepackage{scalerel}
\usepackage{url}
\usepackage{bbold}
\usepackage{amsfonts}
\usepackage{amsmath,amssymb}
\usepackage{bbm}
\usepackage{graphicx}
\usepackage{wrapfig}
\usepackage{mathrsfs} 
\usepackage{algorithm,algorithmic}
\usepackage{array}
\usepackage{times}
\usepackage{url}
\usepackage{cite}
\newcommand{\citet}[1]{\cite{#1}}
\usepackage{upgreek}
\usepackage{float}
\usepackage{longtable}
\usepackage{color}
\usepackage{wasysym}
\usepackage{grffile}
\usepackage{stackengine}

\allowdisplaybreaks

\input{header.tex}

\title{\LARGE \bf Distributed Online System Identification for LTI Systems \\ Using Reverse Experience Replay}

\author{Ting-Jui Chang and Shahin Shahrampour, {\it Senior Member}, {\it IEEE}  
\thanks{T.J. Chang and S. Shahrampour are with the Department of Mechanical and Industrial Engineering, Northeastern University, Boston, MA 02115, USA. 
{\tt\footnotesize email:\{chang.tin,s.shahrampour\}@northeastern.edu}.}
\thanks{This work is supported by NSF ECCS-2136206 Award.}
}

\begin{document}

\maketitle
\thispagestyle{empty}
\pagestyle{empty}

%%%%%%%%%%%%%%%%%%%%%%%%%%%%%%%%%%%%%%%%%%%%%%%%%%%%%%%%%%%%%%%%%%%%%%%%%%%%%%%%
\begin{abstract}
Identification of linear time-invariant (LTI) systems plays an important role in control and reinforcement learning. Both asymptotic and finite-time offline system identification are well-studied in the literature. For online system identification, the idea of stochastic-gradient descent with reverse experience replay (SGD-RER) was recently proposed, where the data sequence is stored in several buffers and the stochastic-gradient descent (SGD) update performs backward in each buffer to break the time dependency between data points. Inspired by this work, we study distributed online system identification of LTI systems over a multi-agent network. We consider agents as identical LTI systems, and the network goal is to jointly estimate the system parameters by leveraging the communication between agents. We propose DSGD-RER, a distributed variant of the SGD-RER algorithm, and theoretically characterize the improvement of the estimation error with respect to the network size. Our numerical experiments certify the reduction of estimation error as the network size grows.   
\end{abstract}

%%%%%%%%%%%%%%%%%%%%%%%%%%%%%%%%%%%%%%%%%%%%%%%%%%%%%%%%%%%%%%%%%%%%%%%%%%%%%%%%
\section{Introduction}
System identification, the process of estimating the {\it unknown} parameters of a dynamical system from the observed input-output sequence, is a classical problem in control, reinforcement learning and time-series analysis. Among this class of problems, learning the transition matrix of a LTI system is a prominent well-studied case, and classical results characterize the asymptotic properties of these estimators \cite{aastrom1971system,ljung1999system,chen2012identification,goodwin1977dynamic}. 

Recently, there has been a renewed interest in the problem of identification of LTI systems, and modern statistical techniques are applied to achieve {\it finite-time} sample complexity guarantees \cite{dean2019sample, jedra2020finite, faradonbeh2018finite, simchowitz2018learning, sarkar2019near}. However, the aforementioned works focus on the offline setup, where the estimator has access to the entire data sequence from the outset. The offline estimator cannot be directly extended to the streaming/online setup, where the system parameters need to be estimated on-the-fly. To this end, the idea of stochastic-gradient descent with reverse experience replay (SGD-RER) is proposed \cite{kowshik2021streaming} to build an online estimator, where the data sequence is stored in several buffers, and the SGD update performs backward in each buffer to break the time dependency between data points. This online method achieves the optimal sample complexity of the offline setup up to log factors.

In this paper, we study the distributed online identification problem for a network of identical LTI systems with unknown dynamics. Each system is modeled as an agent in a multi-agent network, which receives its own data sequence from the underlying system. The goal of this network is to jointly estimate the system parameters by leveraging the communication between agents. We propose DSGD-RER, a distributed version of the online system identification algorithm in \cite{kowshik2021streaming}, where every agent applies SGD in the reverse order and communicates its estimate with its neighbors. We show that this decentralized scheme can improve the estimation error bound as the network size increases, and the simulation results demonstrate this theoretical property.

\vspace{.05cm}
\noindent
{\bf Related Work:} Recently, several works have studied the finite-time properties of LTI system identification. In \cite{dean2019sample}, it is shown that the dynamics of a fully observable system can be recovered from multiple trajectories by a least-squares estimator when the number of trajectories scales linearly with the system dimension. Furthermore, system identification using a single trajectory for stable \cite{jedra2020finite} and unstable \cite{faradonbeh2018finite, simchowitz2018learning, sarkar2019near} LTI systems has been studied. The works of \cite{simchowitz2018learning, jedra2019sample} establish the theoretical lower bound of the sample complexity for fully observable LTI systems. The case of partially observable LTI systems is also addressed for stable \cite{oymak2019non, sarkar2019finite, tsiamis2019finite, simchowitz2019learning} and unstable \cite{zheng2020non} systems. By applying an $\ell_1$-regularized estimator, the work of \cite{fattahi2020learning} improves the dependency of the sample complexity on the system dimension to poly-logarithmic scaling. Contrary to the aforementioned works, which focus on the finite-time analysis of {\it offline} estimators, \citet{kowshik2021streaming} proposes an online estimation algorithm, where the data sequence is split into several buffers, and in each buffer the SGD update is applied in the reverse order to remove the bias due to the time-dependency of data points. As mentioned before, our work builds on \citet{kowshik2021streaming} for {\it distributed} online system identification. 
%\vspace{0.2cm}
%\noindent
%{\bf Non-Linear System Identification:}
Finite-time analysis of system identification has also been studied for non-linear dynamical systems more recently. In the works of \cite{sattar2020non,foster2020learning}, the sample complexity bounds are derived for dynamical systems described via generalized linear models, and \cite{kowshik2021near} further improves the dependence of the complexity bound on the mixing time constant.

\section{Problem Formulation}
\subsection{Notation} 
{\small
\begin{tabular}{|c||l|}
    \hline
    $[m]$ & The set of $\{1,2,\ldots,m\}$ for any integer $m$ \\
    \hline
    $\sigma_{\max}(\Ab)$ & The largest singular value of $\Ab$\\
    \hline
    $\sigma_{\min}(\Ab)$ & The smallest singular value of $\Ab$\\
    \hline
    $\norm{\cdot}$ & Euclidean (spectral) norm of a vector (matrix)\\
    \hline
    $\mathrm{E}[\cdot]$ & The expectation operator\\
    \hline
    $[\Ab]_{ij}$ & The entry in $i$-th row  $j$-th column of $\Ab$\\
    \hline
    % $\Ab\otimes\Bb$ & The Kronecker product of $\Ab$ and $\Bb$\\
    % \hline
    $\Ab\succeq\Bb$ & $(\Ab-\Bb)$ is positive semi-definite\\
    \hline
    $\Ib_d$ & Identity matrix with dimension $d\times d$\\
    \hline
\end{tabular}}

\subsection{Distributed Online  System Identification}
We consider a multi-agent network of $m$ identical LTI systems, where the dynamics of agent $k$ is defined as,
\begin{equation*}
    \xb^k_{t+1} = \Ab\xb^k_{t} + \wb^k_{t},\quad k\in [m].
\end{equation*}
$\Ab\in \mathrm{R}^{d\times d}$ is the {\it unknown} transition matrix, and $\wb^k_{t}$ is the noise sequence generated independently from a distribution with zero mean and finite covariance $\Sigma$. The goal of the agents is to recover the matrix $\Ab$ collaboratively. Though this task can be accomplished by an individual agent (e.g., as in \cite{kowshik2021streaming}), we will show that the collective system identification improves the theoretical  error bound of estimation. 
\begin{assumption}\label{A: Stability}
The considered system is stable in the sense that $\norm{\Ab}<1$.
\end{assumption}
With Assumption \ref{A: Stability}, it can be shown that for any initial state $\xb^k_{0}$, the distribution of $\xb^k_{t}$ converges to a stationary distribution $\pi$ with the covariance matrix $\Gb:=\mathrm{E}_{\xb\sim\pi}[\xb\xb^{\top}] = \sum_{t=0}^{\infty}\Ab^{t}\Sigma\Ab^t$. %The existence of $\pi$ is used for bounding the statistical properties of each state $\xb^k_t$. 
Note that Assumption \ref{A: Stability} is more stringent compared to the standard stability assumption (i.e., $\rho(\Ab)< 1$); however, the analysis is extendable to this case under some modifications of the time horizon (see \cite{kowshik2021streaming} for more details).
\begin{assumption}\label{A: noise sub-gaussian}
For any $\xb\in \mathrm{R}^d$, $\langle\xb, \wb^k_{t}\rangle$ is $C_{\mu}\langle\xb,\Sigma\xb\rangle$ sub-Gaussian, i.e. $$\mathrm{E}_{\wb}[\exp(y\langle\xb, \wb\rangle)]\leq \exp(\frac{y^2}{2}C_{\mu}\langle\xb,\Sigma\xb\rangle),$$
for any $y \in \mathrm{R}$.
\end{assumption}
Sub-Gaussian noise is a standard choice for finite-time analysis of system identification (see e.g., \cite{abbasi2011online, simchowitz2018learning, sarkar2019near}). %and we believe the results can be extended to a more general light-tale distribution assumption.
Also, in some existing works, i.i.d. Gaussian noise is applied to ensure the persistence of excitation (\cite{jedra2019sample, oymak2019non}). In this paper, we assume that $\Sigma$ is positive-definite.

\vspace{.2cm}
\noindent
{\bf Problem Statement:} Departing from the classical offline ordinary least squares (OLS) estimator, the goal is to develop a fully decentralized algorithm, where each agent produces estimates of the unknown system $\Ab$ in an online fashion. Specifically, at each iteration, agent $k$ first updates its estimate based on the current information, and then it communicates its estimate with its neighbors. The communication is modeled via a graph that captures the underlying network topology. As mentioned earlier, the goal of an online distributed system identification algorithm, as opposed to the centralized one, is to leverage the network element to provide a finer estimation guarantee for each agent's estimation error. %In this work, we quantify the estimation error of each agent $k$ by the Euclidean norm of the difference $\norm{\Ab^k - \Ab}$ and characterize an upper bound of this quantity that encapsulates
In this work, we quantify the estimation error of each agent, such that the error bound encapsulates the dependency on the {\it network size} and {\it topology}.

\vspace{.2cm}
\noindent
{\bf Network Structure:} During the estimation process, the agents communicate locally based on a symmetric doubly stochastic matrix $\Pb$, i.e., all elements of $\Pb$ are non-negative and $\sum_{i=1}^m [\Pb]_{ji}=\sum_{j=1}^m [\Pb]_{ji}=1$. Agents $i$ and $j$ communicate with each other if $[\Pb]_{ji}>0$; otherwise $[\Pb]_{ji}=0$. Thus, $\Nc_i:=\{j \in [m]: [\Pb]_{ji}>0\}$ is the neighborhood of agent $i$. The network is assumed to be connected, i.e., for any two agents $i,j\in [m]$, there is a (potentially multi-hop) path from $i$ to $j$. We also assume $\Pb$ has a positive diagonal. Then, there exists a geometric mixing bound for $\Pb$ \cite{liu2008monte}, such that $\sum_{j=1}^m \left|[\Pb^k]_{ji}-1/m\right|\leq \sqrt{m}\beta^k,\:i\in [m],$
where $\beta$ is the second largest singular value of $\Pb$. Agents exchange only local system estimates, and they do not share any other information in the process. The communication is consistent with the structure of $\Pb$. We elaborate on this in the algorithm description.

\begin{figure*}[t!] 
\begin{center}
    \includegraphics[width=1.5\columnwidth]{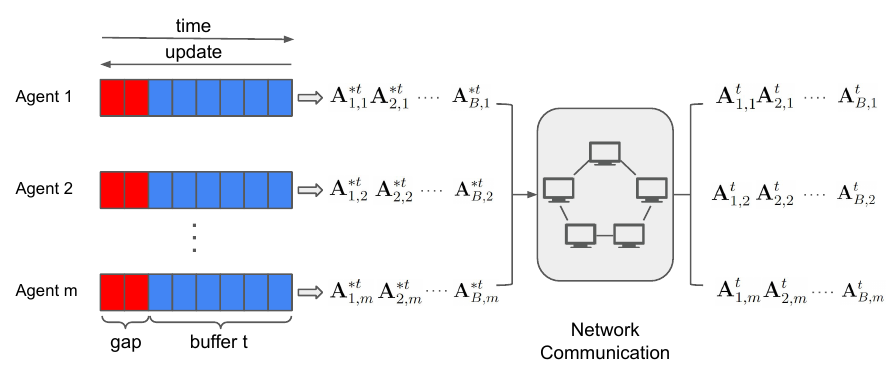}
    \caption{The illustration of the update scheme within one buffer: each agent splits the received data sequence into buffers and applies SGD reversely within each buffer. The estimate of each agent is then updated locally through the network communication. The communication between agents is based on the network structure captured by $\Pb.$}
    \label{fig:illustration of the update}
\end{center}
\end{figure*}

\section{Algorithm and Theoretical Results}
We now lay out the distributed online linear system identification algorithm and provide the theoretical bound for the estimation error. 
\subsection{Offline and Online Settings}
For the identification of LTI systems, it is well-known that the (centralized) OLS estimator $\Ab_{OLS} := \argmin_{\Ab} \sum_{t=0}^{T-1}\norm{\Ab\xb_{t} - \xb_{t+1}}^2$ is statistically optimal. OLS is an offline estimator that can also be implemented in an online fashion (i.e., in the form of recursive least squares) by keeping track of the data covariance matrix and the residual based on the current estimate. %{\rd However, OLS estimator can not be extended to systems like generalized linear models where there is no estimator of a closed form.} 
On the other hand, gradient-based methods provide efficient mechanisms for system identification. In particular, SGD uses the gradient of the current data pair $(\xb_{t+1}, \xb_t)$ to perform the update
\begin{equation*}
    \Ab_{t+1} = \Ab_t - 2\gamma(\Ab_t\xb_t - \xb_{t+1})\xb_t^{\top},
\end{equation*}
where $\gamma$ is the step size. Despite the efficient update, SGD suffers from the time-dependency over the data sequence, which leads to biased estimators. To observe this, unrolling the recursive update rule of SGD, we have 
\begin{equation*}
\begin{split}
    \Ab_t - \Ab &= (\Ab_0 - \Ab)\Pi_{s=0}^{t-1}(\Ib-2\gamma\xb_s\xb_s^{\top})\\
    &+ 2\gamma\sum_{s=0}^{t-1}\wb_s\xb_s^{\top}\Pi_{l=s+1}^{t-1}(\Ib-2\gamma\xb_l\xb_l^{\top}),
\end{split}
\end{equation*}
where in the second term, the dependence of later states on previous noise realizations prevents the estimator from being unbiased even if $\Ab_0$ is initialized with a distribution where $\mathrm{E}[\Ab_0]=\Ab$. To deal with this issue, \cite{kowshik2021streaming} develops the method SGD-RER, which applies SGD in the reverse order of the sequence to break the dependency over time. Suppose that the estimate is updated along the opposite direction of the sequence. In particular, if we have $T$ samples and use the SGD update in the reverse order, the problematic term in the above equation takes the following form $$2\gamma\sum_{k=0}^{t-1}\wb_{T-k}\xb_{T-k}^{\top}\Pi_{l=k+1}^{t-1}(\Ib-2\gamma\xb_{T-l}\xb_{T-l}^{\top}).$$
whose expectation is equal to zero since the later noise realizations are independent of previous states. Evidently, this approach does not work in an online sense (as the entire sequence of $T$ samples has to be collected first). However, we can mimic this approach by dividing the data into smaller buffers and use reverse SGD for each buffer \cite{kowshik2021streaming}. 

\subsection{Distributed SGD with Reverse Experience Replay}
We extend SGD-RER to the distributed case and call it the DSGD-RER algorithm. Each agent splits the sequence of data pairs into several buffers of size $B$, and within each buffer all agents perform SGD in the reverse order collectively based on the network topology. Between two consecutive buffers, $u$ data pairs are discarded in order to decrease the correlation of data in different buffers. The proposed method is outlined in Algorithm \ref{alg:DSGD-RER} and depicted in Fig. \ref{fig:illustration of the update}.

\vspace{.2cm}
\noindent
{\bf Averaged Iterate over Buffers:} To further improve the convergence rate, we average the iterates. Each agent computes as the estimation output the tail-averaged estimate, i.e., the average of the last estimates of all buffers (line 11 of Algorithm \ref{alg:DSGD-RER}).

\vspace{.2cm}
\noindent
{\bf Coupled Process:} Despite the gap between buffers, there is still dependency across data points in various buffers, which makes the analysis challenging. To simplify the analysis, we consider the idea of coupled process \cite{kowshik2021streaming}, where for the state sequence $\{\xb^k_i:i=0,\ldots,T\}$ received by agent $k$, the corresponding coupled sequence $\Tilde{\xb}^k_i$ is defined as follows.
\begin{enumerate}
    \item For each buffer $t$, the starting state $\xb^{k,t}_0$ is independently generated from $\pi$, the stationary distribution of the state. 
    \item The coupled process evolves according to the noise sequence $\wb^{k,t}_i$ of the actual process, such that
    \begin{equation*}
        \Tilde{\xb}^{k,t}_{i+1} = \Ab \Tilde{\xb}^{k,t}_{i} + \wb^{k,t}_i,\: i=0,\ldots,S-1\:(S=B+u).
    \end{equation*}
\end{enumerate}
%Due to the independent buffers, 
We will see in the analysis that it is more straightforward to first compute the estimation error based on the coupled process and then quantify the excessive error introduced by replacing the coupled process with the actual one. Note that based on the dynamics of the coupled process, the excessive error is related to the gap size $u$.

\begin{algorithm}[tb]
   \caption{DSGD-RER}
   \label{alg:DSGD-RER}
\begin{algorithmic}[1]
   \STATE {\bfseries Require:} number of agents $m$, doubly stochastic matrix $\Pb\in \mathrm{R}^{m\times m}$, step size $\gamma$, buffer size $B$, gap size $u$,  time horizon $T$, the parameter $\tau$.
   
  \STATE {\bf Initialize:} $\Ab^{0}_{0,k}$ is initialized as a zero matrix for all agents $k\in [m]$. The number of buffers $N = T/S$, where $S=B+u$.
  
   \FOR{$t=0,1,\ldots,N-1$}
        \STATE Each agent $k$ collects its data sequence of buffer $t$, $\{\xb^{k,t}_{0}, \xb^{k,t}_{1},\ldots,\xb^{k,t}_{S-1}\}$, where $\xb^{k,t}_{i}:=\xb^k_{S\cdot t+i}$, and we also define $\xb^{k,t}_{-i}:=\xb^{k,t}_{(S-1)-i}$.
        \FOR{$i=0,1,\ldots,B-1$}
            \FOR{$k=1,2\ldots,m$}
                \STATE$\Ab^{*t}_{i+1,k} = \Ab^{t}_{i,k} - 2\gamma(\Ab^{t}_{i,k}\xb^{k,t}_{-i} - \xb^{k,t}_{-(i-1)}){(\xb^{k,t}_{-i})}^\top$
            \ENDFOR    
            \STATE Each agent $k$ communicates with its neighbors based on $\Pb$ to update its estimate:
            \begin{equation*}
                \Ab^{t}_{i+1,k} = \sum_{j\in \Nc_i}[\Pb{^\tau}]_{jk}\Ab^{*t}_{i+1,j}.
            \end{equation*}
        \ENDFOR
        \STATE For each agent $k$, compute the tail-averaged estimate until the current buffer as $\hat{\Ab}^{k}_{0,t} = \frac{1}{t+1}\sum_{s=0}^t\Ab^{s}_{B,k}$ and define $\Ab^{t+1}_{0,k} = \Ab^{t}_{B,k}$.
   \ENDFOR
\end{algorithmic}
\end{algorithm}

\subsection{Theoretical Guarantee}
In this section, we present our main theoretical result. By running Algorithm \ref{alg:DSGD-RER} with specified hyper-parameters, we show that for the estimation error upper bound (of any agent), the term corresponding to the leading term in \cite{kowshik2021streaming} can be improved by increasing the network size. There exists a (high probability) upper bound $R$ for the norm of state, such that $\norm{\xb^{k,t}_{i}}\leq \sqrt{R}$, which is also one of the parameters to be tuned. 
\begin{theorem}\label{T: D-SGDRER}
Let Assumptions \ref{A: Stability}-\ref{A: noise sub-gaussian} hold and consider the following hyper-parameter setup:
    \begin{enumerate}
        \item  $d\leq O(\log(T))$.
        \item $\gamma RB < \frac{1}{8}$.
        \item $c_1\gamma B\sigma_{\min}(\Gb)\leq \frac{1}{2}$, where $c_1$ is a problem-dependent constant.
        \item $R = \Omega(C_{\mu}\sigma_{\max}(\Gb)\log (T)),\\ B = u = \Theta(\sqrt{\frac{T}{\log(T)}}),\\ \gamma = \Theta(\frac{1}{\sqrt{T\log T}}\frac{\log \sqrt{T\log(T)}}{\log(T)}),\text{ and}\\
        \tau = \Theta(\log(T))$.
    \end{enumerate}
    Then, by applying Algorithm \ref{alg:DSGD-RER}, with probability at least $1-(\frac{5m}{T^{\rho}}+\frac{3}{T^{v}})$ where $\rho$ and $v$ are some positive constants, the estimation error of the tail-averaged estimate of agent $k\in [m]$ over all buffers is bounded as follows
    \begin{equation}\label{eq: Main Theorem}
        \begin{split}
            &\norm{\hat{\Ab}^{k}_{0,N-1} - \Ab}\\
            \leq &2\sqrt{m}\beta^{\tau}\gamma RT^2 + (16\gamma^2 R^2T^2 + 8\gamma RT)\norm{\Ab^u}\\
            + &\sqrt{\frac{8\gamma C_{\mu}\sigma_{\max}(\Sigma)(c_6d + v\log T )(1+2\alpha)}{Nm}} + \frac{\zeta\norm{\Ab_0-\Ab}}{N}\\
            + &\frac{4\norm{\Ab_0-\Ab}}{c_1\gamma\sigma_{\min}(\Gb)NB}\exp \big(-\frac{c_1\gamma B\sigma_{\min}(\Gb)}{2}(\lceil\zeta\rceil+1)\big),
        \end{split}
    \end{equation}
    where $\zeta:=\max\{\frac{c_4d}{c_3} + \log(NT^v)/c_3 - 1, c_2\}$, $\alpha:= c_5(\zeta + \frac{1}{\gamma B\sigma_{\min}(\Gb)})$ and $c_1,\ldots,c_5$ are constants. 
\end{theorem}
\vspace{0.2cm}
\noindent
\begin{remark}
Comparing \eqref{eq: Main Theorem} and the estimation error upper bound of the centralized SGD-RER in \cite{kowshik2021streaming}, we can observe that for
the dominant term in \cite{kowshik2021streaming}, the corresponding term in \eqref{eq: Main Theorem} is $\sqrt{\frac{8\gamma C_{\mu}\sigma_{\max}(\Sigma)(c_6d + v\log T )(1+2\alpha)}{Nm}}$, which shows that the contributed error can be improved by increasing the network size $m$.
\end{remark}

\section{Numerical Experiments}

\begin{figure}[t!] 
    \includegraphics[width=.95\columnwidth]{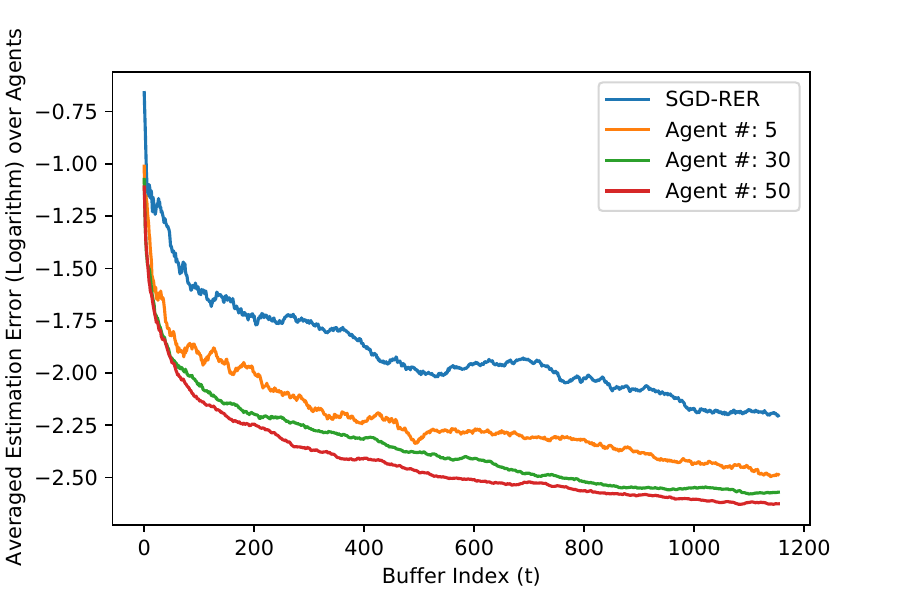}
    \caption{The plot shows that the estimation error of DSGD-RER improves as the network size increases. Also, DSGD-RER outperforms it centralized counter part SGD-RER.}
    \label{fig:net size}
\end{figure}

%We now provide numerical simulations verifying the theoretical guarantee of our algorithm.

\vspace{0.2cm}
\noindent
\textbf{Experiment Setup:} We consider a network of LTI systems, where the transition matrix $\Ab\in\mathrm{R}^{d\times d}$ has the form $\Ub \Lambda \Ub^{\top}$. $\Ub$ is a randomly generated orthogonal matrix and $\Lambda$ is a diagonal matrix with two of its diagonal entries equal to $0.9$ and the rest equal to $0.3$. For the step size, we set $\gamma_k=\frac{1}{2R_k}$, where $R_k$ is estimated as the sum of the norms of the first $\lfloor2 \log T\rfloor$ samples of agent $k$. We also set the other hyper-parameters as follows: $T = 10^7$, $u = \sqrt{\frac{T}{\log T}}$, $B = 10u$, $\tau=1$ and $d=5$. The starting state $\xb^k_0=0$, and the noise follows the standard Gaussian distribution $\mathcal{N}(0,\Ib)$ $\forall k\in[m]$.

\vspace{0.2cm}
\noindent
\textbf{Networks:} We examine the dependence of the estimation error to {\bf 1)} the network size and {\bf 2)} the network topology. For the former, we consider the centralized SGD-RER \cite{kowshik2021streaming} and our distributed algorithm for $2$-cyclic graph with various agent numbers $m\in \{5,30,50\}$. For the latter, we look at the performance of a $5$-agent network with the following topologies: {\bf a)} $\Pb=\Ib$ (Net A), fully disconnected; {\bf b)} $\Pb$ captures a $2$-cyclic graph, where each agent has a self-weight of $0.3$ and assigns the weight of $0.35$ to each of its two neighbors (Net B); {\bf c)} $\Pb = \frac{1}{m}{\mathbf{1}}{\mathbf{1}}^{\top}$ (Net C), the fully connected network.

\begin{figure}[t!] 
    \includegraphics[width=.95\columnwidth]{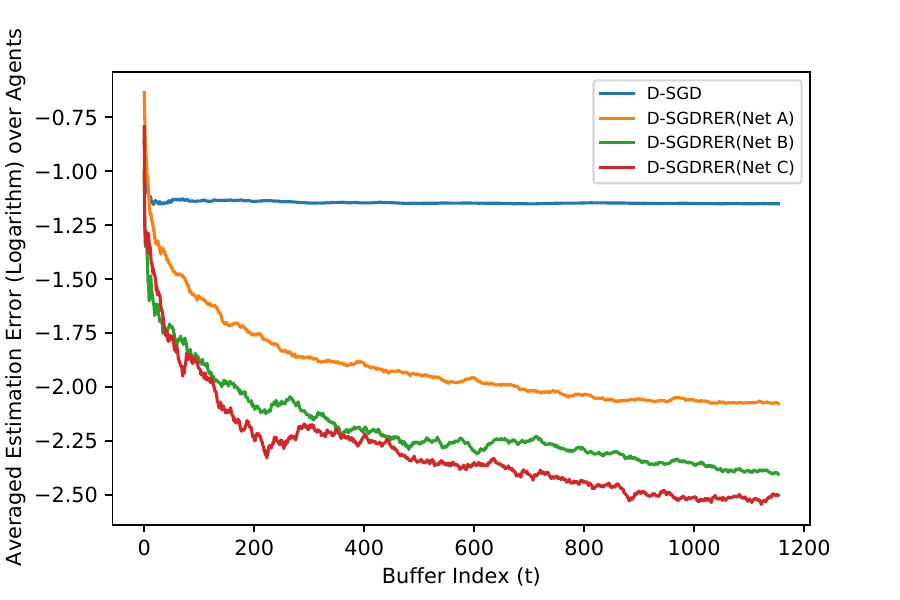}
    \caption{The averaged estimation error over time for different network topologies. The better the network is connected, the lower the estimation error is. Vanilla D-SGD (even in a fully connected network) is not competitive due to the correlation of data points.}
    \label{fig:topology}
\end{figure}

\vspace{0.2cm}
\noindent
\textbf{Performance:} From Fig. \ref{fig:net size}, we verify the dependency of the estimation precision on the network size. For a larger network size, the error of the tail-averaged estimate is smaller. Furthermore, all decentralized schemes outperform centralized SGD-RER \cite{kowshik2021streaming}. To see the influence of network topology, notice that the value of $\beta$ for different networks is ordered as $\text{Net A}>\text{Net B}>\text{Net C}$. We can see in Fig. \ref{fig:topology} that smaller $\beta$ results in a better estimation error. Furthermore, to see the impact of reverse estimation, we also plot the estimation error of applying vanilla distributed SGD in a fully connected network. We can see that the error does not shrink over time as the estimator is biased, so the reverse update process in DSGD-RER is critical.

\section{Conclusion}
In this work, we considered the distributed online system identification problem for a network of identical LTI systems. We proposed a distributed online estimation algorithm and characterized the estimation error with respect to the network size and topology. We indeed observed in numerical experiments that larger network size as well as better connectivity can improve the estimation performance. %For future directions, this online estimation method can be combined with model-based control techniques to build real-time adaptive controllers. 
For future directions, this online estimation method can be combined with decentralized control techniques to build decentralized, real-time adaptive controllers as explored in \cite{chang2021distributed,chang2021regret}. Another potential direction is to extend the technique for identification of non-linear dynamical systems.

\section{Appendix}
For the complete proof of claims, see the longer version of the paper \cite{chang2022distributed}, where all necessary lemmas are included. 

\addtolength{\textheight}{-0cm}  
\subsection{Notations and Proof Sketch}\label{S: Notations and proof sketch}
To analyze the distributed estimation error, we decompose it into two parts: (I) The estimation error from applying the fully-connected network $\Pb_{avg}=\frac{1}{m}{\mathbf{1}}{\mathbf{1}}^{\top}$ as the communication scheme; (II) The difference between estimates derived from $\Pb$ and $\Pb_{avg}$. To analyze (I), we apply the idea of coupled process: we consider $\Tilde{\Ab}^{t,avg}_{i,k}$, agent $k$ estimate based on the coupled process and $\Pb_{avg}$. From the estimate initialization and the update scheme, we know $\Tilde{\Ab}^{t,avg}_{i,1}=\Tilde{\Ab}^{t,avg}_{i,2}=\cdots=\Tilde{\Ab}^{t,avg}_{i,m}$, which we denote as $\Tilde{\Ab}^{t,avg}_{i}$, coming with the recursive update rule:
\begin{equation}\label{eq:tilde}
    \Tilde{\Ab}^{t,avg}_{i+1} = \Tilde{\Ab}^{t,avg}_{i} - \frac{2\gamma}{m}\sum_{k=1}^m (\Tilde{\Ab}^{t,avg}_{i}\Tilde{\xb}^{k,t}_{-i} - \Tilde{\xb}^{k,t}_{-(i-1)})\Tilde{\xb}^{k,t\top}_{-i}.
\end{equation}
%In Section \ref{S: Bounds of tilde A}, we provide error bounds for $\norm{\Tilde{\Ab}^{t,avg}_{i} - \Ab}$ and $\norm{\Tilde{\Ab}^{t,avg}_{i} - \Ab^{t,avg}_{i}}$, which together with the results for (II) help with deriving the main result.
We use the following notations for the rest of the proof:
\begin{equation*}
\begin{split}
    \tilde{\Pb}^{t, avg}_i &:= \Ib - 2\gamma\frac{\sum_{k=1}^m \Tilde{\xb}^{k,t}_i\Tilde{\xb}^{k,t\top}_i}{m},\\
    \tilde{\Hb}^{t,avg}_{i,j}&:=\begin{cases}\prod_{s=i}^j \tilde{\Pb}^{t, avg}_{-s},\:&i\leq j\\ \Ib,\:&i>j\end{cases}
\end{split}
\end{equation*}
where the index $-j:=(S-1)-j$, and 
\begin{equation*}
\begin{split}
    \xb^{k,t}_{-i} &:= \xb^{k,t}_{(S-1)-i},\quad 0\leq i\leq S-1,\\
    % \xb^{t.avg}_{-i} &= \frac{1}{m}\sum_{k=1}^m\xb^{k,t}_{-i},\\
    \Dc^t_{-j} &:= \Big\{\norm{\xb^{k,t}_{-i}}\leq \sqrt{R}:\:\forall k,\:j\leq i \leq B-1\Big\},\\
    \Tilde{\Dc}^t_{-j} &:= \Big\{\norm{\Tilde{\xb}^{k,t}_{-i}}\leq \sqrt{R}:\:\forall k,\: j\leq i \leq B-1\Big\},\\
    \Dc^{s,t}&:= \cap^t_{r=s}\Dc^r_{-0},\:s\leq t, \Tilde{\Dc}^{s,t}:=\cap^t_{r=s}\Tilde{\Dc}^r_{-0},\:s\leq t ,\\
    \hat{\Dc}^{s,t} &:= \Dc^{s,t}\cap \Tilde{\Dc}^{s,t}.
\end{split}
\end{equation*}
The estimate $\Tilde{\Ab}^{avg}$ consists of a bias term and a variance term, which have the following expressions (see equations 18-19 in \cite{kowshik2021streaming}): 
\begin{align*}
    \Tilde{\Ab}^{t,avg}_B - \Ab &= (\Tilde{\Ab}^{t,avg}_{B,bias} - \Ab) + \Tilde{\Ab}^{t,avg}_{B,var},\\
    (\Tilde{\Ab}^{t,avg}_{B,bias} - \Ab) &= (\Ab_0-\Ab)\prod_{s=0}^t \tilde{\Hb}^{s,avg}_{0,B-1},    
\end{align*}{\small
\begin{equation*}
\begin{split}
    &\Tilde{\Ab}^{t,avg}_{B,var}=\\
    &2\gamma\sum_{r=0}^t\sum_{j=0}^{B-1}\big(\frac{\sum_{k=1}^m \wb^{k,t-r}_{-j}\Tilde{\xb}^{k,t-r\top}_{-j}}{m}\big)\tilde{\Hb}^{t-r,avg}_{j+1,B-1}\prod_{s=r-1}^0 \tilde{\Hb}^{t-s,avg}_{0,B-1}.
\end{split}    
\end{equation*}}

\subsection{Proof of the Main Theorem}
First, we decompose the estimation error into several terms as follows:
\begin{equation}\label{eq1: Main theorem}
\begin{split}
    &\norm{\hat{\Ab}^{k}_{0,N-1} - \Ab}\\
    %\leq &\norm{\hat{\Ab}^{k}_{0,N-1} - \hat{\Ab}^{avg}_{0,N-1} + \hat{\Ab}^{avg}_{0,N-1} - \Ab}\\
    \leq &\norm{\hat{\Ab}^{k}_{0,N-1} - \hat{\Ab}^{avg}_{0,N-1}} + \norm{\hat{\Ab}^{avg}_{0,N-1} - \hat{\Tilde{\Ab}}^{avg}_{0,N-1}}\\
    + &\norm{\hat{\Tilde{\Ab}}^{var,avg}_{0,N-1}} + \norm{\hat{\Tilde{\Ab}}^{bias,avg}_{0,N-1}-\Ab}.
\end{split}
\end{equation}
{\bf (I)} For the term $\norm{\hat{\Ab}^{k}_{0,N-1} - \hat{\Ab}^{avg}_{0,N-1}}$:\\
From Lemma \ref{L: Bound on the difference between estimates from different networks}, when the event of $\Dc^{0,N-1}$ holds, we have
\begin{equation*}
\begin{split}
    \norm{\hat{\Ab}^{k}_{0,N-1} - \hat{\Ab}^{avg}_{0,N-1}} &\leq \frac{1}{N}\sum_{t=0}^{N-1}\norm{\Ab^{t}_{B,k} - \Ab^{t,avg}_{B}}\\
     &\leq 2\sqrt{m}\beta^{\tau}\gamma RT^2.
\end{split}
\end{equation*}
From Lemma 9 in \cite{kowshik2021streaming}, there exists a constant $\rho>0$ such that if $R=\Omega(C_{\mu}\sigma_{\max}(\Gb)\log T)$, $\mathcal{P}(\Dc^{0,N-1}) \geq 1-\frac{m}{T^{\rho}}$, from which we have
\begin{equation}\label{eqI: Main theorem}
\resizebox{0.95\hsize}{!}{$
    \mathcal{P}\left(\norm{\hat{\Ab}^{k}_{0,N-1} - \hat{\Ab}^{avg}_{0,N-1}}\geq 2\sqrt{m}\beta^{\tau}\gamma RT^2\right)\leq \frac{m}{T^{\rho}}.
$}
\end{equation}
{\bf (II)} For the term $\norm{\hat{\Ab}^{avg}_{0,N-1} - \hat{\Tilde{\Ab}}^{avg}_{0,N-1}}$:\\
Based on the expression of the tail-averaged estimate and Lemma \ref{L: Coupled bound}, under the event of $\hat{\Dc}^{0,N-1}$, we have
\begin{equation*}
\begin{split}
    \norm{\hat{\Ab}^{avg}_{0,N-1} - \hat{\Tilde{\Ab}}^{avg}_{0,N-1}}\leq &\frac{1}{N}\sum_{t=0}^{N-1}\norm{\Ab^{t,avg}_{B} - \Tilde{\Ab}^{t,avg}_{B}}\\
    \leq &(16\gamma^2 R^2T^2 + 8\gamma RT)\norm{\Ab^u}.
\end{split}
\end{equation*}
Based on Lemma 9 in \cite{kowshik2021streaming}, there exists a constant $\rho>0$ such that if $R=\Omega(C_{\mu}\sigma_{\max}(\Gb)\log T)$, $\mathcal{P}(\hat{\Dc}^{0,N-1}) \geq 1-\frac{2m}{T^{\rho}}$. Therefore, we conclude that
\begin{equation}\label{eq2: Main theorem}
\resizebox{0.99\hsize}{!}{$
\begin{split}
    \mathcal{P}\left(\norm{\hat{\Ab}^{avg}_{0,N-1} - \hat{\Tilde{\Ab}}^{avg}_{0,N-1}}\geq (16\gamma^2 R^2T^2 + 8\gamma RT)\norm{\Ab^u}\right)\leq \frac{2m}{T^{\rho}}.
\end{split}
$}
\end{equation}
{\bf (III)} For the term $\norm{\hat{\Tilde{\Ab}}^{var,avg}_{0,N-1}}$:\\
By applying Theorem \ref{T: Bound of the tail average of the variance term} with the probability upper bound equal to $\frac{1}{T^v}$ where $v$ is some positive constant, we have that under the event of $\Tilde{\Mc}^{0,N-1}\cap \Tilde{\Dc}^{0,N-1}$ (the definition of $\tilde{\Mc}$ can be found in \eqref{eq: Definition of event M}), with probability at least $1-\frac{1}{T^v}$,
\begin{equation}\label{eq3: Main theorem}
    \norm{\hat{\Tilde{\Ab}}^{var,avg}_{0,N-1}}\leq \sqrt{\frac{8\gamma C_{\mu}\sigma_{\max}(\Sigma)(c_6d + v\log T )(1+2\alpha)}{Nm}}.
\end{equation}
Based on Lemma 9 in \cite{kowshik2021streaming}, there exists a constant $\rho>0$ such that if $R=\Omega(C_{\mu}\sigma_{\max}(\Gb)\log T)$, $\mathcal{P}(\Tilde{\Dc}^{0,N-1}) \geq 1-\frac{m}{T^{\rho}}$. Under the event of $\Tilde{\Dc}^{0,N-1}$, by setting $\delta$ in Lemma \ref{L: Probability bound of the F term} properly, we have
\begin{equation*}
    \mathcal{P}(\Tilde{\Mc}^{0,N-1}\cap \Tilde{\Dc}^{0,N-1})\geq  1-(\frac{1}{T^v} + \frac{m}{T^{\rho}}).
\end{equation*}
Combining above with \eqref{eq3: Main theorem}, we conclude that
\begin{equation}\label{eq4: Main theorem}
\resizebox{0.95\hsize}{!}{$
\begin{split}
    &\mathcal{P}\left(\norm{\hat{\Tilde{\Ab}}^{var,avg}_{0,N-1}}\geq \sqrt{\frac{8\gamma C_{\mu}\sigma_{\max}(\Sigma)(c_6d + v\log T )(1+2\alpha)}{Nm}}\right)\\
    &\leq (\frac{2}{T^v} + \frac{m}{T^{\rho}}).
\end{split}
$}
\end{equation}
{\bf (IV)} For the term $\norm{\hat{\Tilde{\Ab}}^{bias,avg}_{0,N-1}-\Ab}$:\\
Based on the expression of the bias term in Section \ref{S: Notations and proof sketch}, we have that
\begin{equation}\label{eq5: Main theorem}
\begin{split}
    \norm{\hat{\Tilde{\Ab}}^{bias,avg}_{0,N-1}-\Ab} = &\norm{\frac{1}{N}\sum_{t=0}^{N-1}(\Tilde{\Ab}^{t,avg}_{B,bias}-\Ab)}\\
    \leq &\frac{\norm{\Ab_0-\Ab}}{N}\sum_{t=0}^{N-1}\norm{\prod_{s=0}^t\tilde{\Hb}^{s,avg}_{0,B-1}}.
\end{split}
\end{equation}
From the result of Lemma \ref{L: Probability bound 2 of the H product} with $\zeta:=\max\{\frac{c_4d}{c_3} + \log(\frac{N}{\delta})/c_3 - 1, c_2\}$, we have, for all $t\geq \zeta$
\begin{equation}\label{eq6: Main theorem}
    \norm{\prod_{s=0}^t\tilde{\Hb}^{s,avg}_{0,B-1}}\leq 2\big(1-\gamma B\sigma_{\min}(\Gb)\big)^{\frac{c_1}{2}(t+1)},
\end{equation}
with probability at least $(1-\delta)$ under the event of $\Tilde{\Dc}^{0,N-1}$.
Based on Lemma \ref{L: Probability bound 2 of the H product}, under the event of $\Tilde{\Dc}^{0,N-1}$ we also have for all $t<\zeta$,
\begin{equation}\label{eq7: Main theorem}
    \norm{\prod_{s=0}^t\tilde{\Hb}^{s,avg}_{0,B-1}}\leq 1 \text{ almost surely}.
\end{equation}
Based on \eqref{eq5: Main theorem}, \eqref{eq6: Main theorem} and \eqref{eq7: Main theorem}, we have, under the event of $\Tilde{\Dc}^{0,N-1}$ with probability at least $(1-\delta)$,
\begin{equation}\label{eq8: Main theorem}
\resizebox{0.9\hsize}{!}{$
\begin{split}
    &\norm{\hat{\Tilde{\Ab}}^{bias,avg}_{0,N-1}-\Ab}\\
    \leq &\frac{\norm{\Ab_0-\Ab}}{N}\left(\sum_{t=0}^{\lceil\zeta\rceil-1}\norm{\prod_{s=0}^t\tilde{\Hb}^{s,avg}_{0,B-1}} + \sum_{t=\lceil\zeta\rceil}^{N-1}\norm{\prod_{s=0}^t\tilde{\Hb}^{s,avg}_{0,B-1}}\right)\\
    \leq &\frac{\norm{\Ab_0-\Ab}}{N}\left(\zeta + \sum_{t=\lceil\zeta\rceil}^{N-1}2\big(1-\gamma B\sigma_{\min}(\Gb)\big)^{\frac{c_1}{2}(t+1)}\right)\\
    \leq &\frac{\norm{\Ab_0-\Ab}}{N}2\big(1-\gamma B\sigma_{\min}(\Gb)\big)^{\frac{c_1}{2}}\sum_{t=\lceil\zeta\rceil}^{N-1} \big(1-\frac{c_1\gamma B\sigma_{\min}(\Gb)}{2}\big)^t\\
    + &\frac{\zeta\norm{\Ab_0-\Ab}}{N}\\
    \leq &\frac{2(1-\frac{c_1\gamma B\sigma_{\min}(\Gb)}{2})^{\lceil\zeta\rceil}}{c_1\gamma B\sigma_{\min}(\Gb)}\frac{\norm{\Ab_0-\Ab}}{N}2\big(1-\gamma B\sigma_{\min}(\Gb)\big)^{\frac{c_1}{2}}\\
    + &\frac{\zeta\norm{\Ab_0-\Ab}}{N}\\
    \leq &\frac{4\norm{\Ab_0-\Ab}}{c_1\gamma\sigma_{\min}(\Gb)NB}\exp \big(-\frac{c_1\gamma B\sigma_{\min}(\Gb)}{2}(\lceil\zeta\rceil+1)\big)\\
    + &\frac{\zeta\norm{\Ab_0-\Ab}}{N}
\end{split}
$}
\end{equation}
By setting $\delta$ as $\frac{1}{T^v}$, based on \eqref{eq8: Main theorem} and the fact that $\mathcal{P}(\Tilde{\Dc}^{0,N-1})\geq 1-\frac{m}{T^{\rho}}$, we have
\begin{equation}\label{eq9: Main theorem}
\resizebox{1\hsize}{!}{$
\begin{split}
    &\mathcal{P}\left(\norm{\hat{\Tilde{\Ab}}^{bias,avg}_{0,N-1}-\Ab}\geq \frac{4\norm{\Ab_0-\Ab}}{c_1\gamma\sigma_{\min}(\Gb)NB}\exp \big(-\frac{c_1\gamma B\sigma_{\min}(\Gb)}{2}(\lceil\zeta\rceil+1)\big)
    + \frac{\zeta\norm{\Ab_0-\Ab}}{N}\right)\\
    &\leq (\frac{1}{T^v} + \frac{m}{T^{\rho}}).
\end{split}
$}
\end{equation}
Applying the results of \eqref{eqI: Main theorem}, \eqref{eq2: Main theorem}, \eqref{eq4: Main theorem} and \eqref{eq9: Main theorem} on \eqref{eq1: Main theorem}, with probability at least $1-(\frac{5m}{T^{\rho}} + \frac{3}{T^v})$ we have 
\begin{equation}\label{eq10: Main theorem}
\begin{split}
    &\norm{\hat{\Ab}^{k}_{0,N-1} - \Ab}\\
    \leq &2\sqrt{m}\beta^{\tau}\gamma RT^2 + (16\gamma^2 R^2T^2 + 8\gamma RT)\norm{\Ab^u}\\
    + &\sqrt{\frac{8\gamma C_{\mu}\sigma_{\max}(\Sigma)(c_6d + v\log T )(1+2\alpha)}{Nm}} + \frac{\zeta\norm{\Ab_0-\Ab}}{N}\\
    + &\frac{4\norm{\Ab_0-\Ab}}{c_1\gamma\sigma_{\min}(\Gb)NB}\exp \big(-\frac{c_1\gamma B\sigma_{\min}(\Gb)}{2}(\lceil\zeta\rceil+1)\big).
\end{split}
\end{equation}

%%%%%%%%%%%%%%%%%%%%%%%%%%%%%%%%%%%%%%%%%%%%%%%%%%%%%%%%%%%%%%%%%%%%%%%%%%%%%%%%

\bibliographystyle{IEEEtran}
\bibliography{references}

%%%%%%%%%%%%%%%%%%%%%%%%%%%%%%%%%%%%%%%%%%%%%%%%%%%%%%%%%%%%%%%%%%%%%%%%%%%%%%%%

\newpage

%\renewcommand{\appendixname}{Supplementary Material}

%\appendix

\section{Supplementary Material}

\subsection{Bounds of the Estimation Error of $\Tilde{\Ab}^{avg}$ and $\Ab^{avg}$}\label{S: Bounds of tilde A}
\begin{lemma}\label{L:Bounding the exponential moment of the H product}
    Suppose that Assumption \ref{A: noise sub-gaussian} and $\gamma R< 1/2$ hold. Then, $\forall t\geq 0$ and for any $\lambda \in \mathrm{R}$ and $\xb,\yb \in \mathrm{R}^d$, we have {\small 
    \begin{equation*}
    \begin{split}
        &\mathrm{E}\big[\exp(\frac{\gamma}{m} \lambda^2 C_{\mu}\xb^{\top}\Sigma\xb\langle\yb, \tilde{\Hb}^{t,avg \top}_{0,B-1}\tilde{\Hb}^{t,avg}_{0,B-1}\yb\rangle + \lambda \xb^{\top}\Delta^t \yb)|\Tilde{\Dc}^t_{-0}\big]\\
        \leq &\frac{\exp (\frac{\gamma}{m}\lambda^2 C_{\mu}\xb^{\top}\Sigma\xb\norm{\yb}^2)}{\mathcal{P}(\Tilde{\Dc}^t_{-0})}.
    \end{split}
    \end{equation*}}
\end{lemma}
\begin{proof}
    For the ease of analysis, we have the following notations:
    \begin{equation*}
    \begin{split}
        \Delta^t&:= 2\gamma\sum_{j=0}^{B-1}\big(\frac{\sum_{k=1}^m \wb^{k,t}_{-j}\Tilde{\xb}^{k,t\top}_{-j}}{m}\big)\tilde{\Hb}^{t,avg}_{j+1,B-1},\\
        \Delta^t_{-i}&:= 2\gamma\sum_{j=i}^{B-1}\big(\frac{\sum_{k=1}^m \wb^{k,t}_{-j}\Tilde{\xb}^{k,t\top}_{-j}}{m}\big)\tilde{\Hb}^{t,avg}_{j+1,B-1},\\
        \Xi^t_0&:= \frac{\gamma}{m} \lambda^2 C_{\mu}\xb^{\top}\Sigma\xb\langle\yb, \tilde{\Hb}^{t,avg \top}_{0,B-1}\tilde{\Hb}^{t,avg}_{0,B-1}\yb\rangle,\\
        \Xi^t_i&:= \frac{\gamma}{m} \lambda^2 C_{\mu}\xb^{\top}\Sigma\xb\langle\yb, \tilde{\Hb}^{t,avg \top}_{i,B-1}\tilde{\Hb}^{t,avg}_{i,B-1}\yb\rangle.
    \end{split}
    \end{equation*}
    From the notation above and  $\gamma R< 1/2$, under the event  $\Tilde{\Dc}^t_{-0}$ we have
    \begin{equation}\label{eq1: Bounding the exponential moment of the H product}
    \resizebox{0.95\hsize}{!}{$
    \begin{split}
        &\Xi^t_i + \sum_{k=1}^m \frac{2\gamma^2\lambda^2 C_{\mu}}{m^2}\langle\xb,\Sigma\xb\rangle\langle\yb,\tilde{\Hb}^{t,avg\top}_{i+1,B-1}\Tilde{\xb}^{k,t}_{-i}\Tilde{\xb}^{k,t\top}_{-i}\tilde{\Hb}^{t,avg}_{i+1,B-1}\yb\rangle\\
        =&\frac{\gamma}{m}\lambda^2 C_{\mu}\xb^{\top}\Sigma\xb\langle\yb,\tilde{\Hb}^{t,avg \top}_{i+1,B-1}\big(\Tilde{\Pb}^{t,avg\top}_{-i}\Tilde{\Pb}^{t,avg}_{-i} \big)\tilde{\Hb}^{t,avg}_{i+1,B-1}\yb\rangle\\
        +&\frac{\gamma}{m}\lambda^2 C_{\mu}\xb^{\top}\Sigma\xb\langle\yb,\tilde{\Hb}^{t,avg \top}_{i+1,B-1}\big( \frac{2\gamma}{m}\sum_{k=1}^m(\Tilde{\xb}^{k,t}_{-i}\Tilde{\xb}^{k,t\top}_{-i})\big)\tilde{\Hb}^{t,avg}_{i+1,B-1}\yb\rangle\\
        \leq&\frac{\gamma}{m}\lambda^2 C_{\mu}\xb^{\top}\Sigma\xb\langle\yb,\tilde{\Hb}^{t,avg \top}_{i+1,B-1}\tilde{\Hb}^{t,avg}_{i+1,B-1}\yb\rangle=\Xi^t_{i+1},
    \end{split}    
    $}
    \end{equation}
    where the inequality is due to the following
    \begin{equation*}
    \begin{split}
        &\Tilde{\Pb}^{t,avg\top}_{-i}\Tilde{\Pb}^{t,avg}_{-i} + \frac{2\gamma}{m}\sum_{k=1}^m(\Tilde{\xb}^{k,t}_{-i}\Tilde{\xb}^{k,t\top}_{-i}) \\
        =&\big(\Ib -\frac{2\gamma}{m}\sum_{k=1}^m\Tilde{\xb}^{k,t}_{-i}\Tilde{\xb}^{k,t\top}_{-i} + \frac{4\gamma^2 }{m^2}\sum_{k=1}^m\sum_{s=1}^m\Tilde{\xb}^{k,t}_{-i}\Tilde{\xb}^{k,t\top}_{-i}\Tilde{\xb}^{s,t}_{-i}\Tilde{\xb}^{s,t\top}_{-i} \big)\\
        \preceq&\left(\Ib -\frac{2\gamma}{m}\sum_{k=1}^m\Tilde{\xb}^{k,t}_{-i}\Tilde{\xb}^{k,t\top}_{-i} + \frac{4\gamma^2 R}{m}\sum_{k=1}^m\Tilde{\xb}^{k,t}_{-i}\Tilde{\xb}^{k,t\top}_{-i} \right)\preceq \Ib.
    \end{split}
    \end{equation*}
    Knowing that $\Delta^t_{-i} = 2\gamma \big(\frac{\sum_{k=1}^m \wb^{k,t}_{-i}\Tilde{\xb}^{k,t\top}_{-i}}{m}\big)\tilde{\Hb}^{t,avg}_{i+1,B-1} + \Delta^t_{-(i+1)}$ and defining 
    \begin{equation*}
    \resizebox{0.99\hsize}{!}{$
        \varkappa^t_i:=\sum_{k=1}^m \frac{2\gamma^2\lambda^2 C_{\mu}}{m^2}\langle\xb,\Sigma\xb\rangle\langle\yb,\tilde{\Hb}^{t,avg\top}_{i+1,B-1}\Tilde{\xb}^{k,t}_{-i}\Tilde{\xb}^{k,t\top}_{-i}\tilde{\Hb}^{t,avg}_{i+1,B-1}\yb\rangle,
    $}    
    \end{equation*}
     we get {\small
    \begin{equation}\label{eq2: Bounding the exponential moment of the H product}
    \begin{split}
        &\mathrm{E}\left[\exp (\Xi^t_0 + \lambda \langle\xb,\Delta^t\yb\rangle)\middle|\Tilde{\Dc}^t_{-0}\right]\\
        =&\frac{\mathrm{E}\left[\exp (\Xi^t_0 + \lambda \langle\xb,\Delta^t_{-0}\yb\rangle)\mathbf{1}(\Tilde{\Dc}^t_{-0})\right]}{\mathcal{P}(\Tilde{\Dc}^t_{-0})}\\
        =&\frac{\mathrm{E}\left[\exp (\Xi^t_0 + \lambda \langle\xb,(\Delta^t_{-0} - \Delta^t_{-1})\yb\rangle +\lambda\langle\xb,\Delta^t_{-1}\yb\rangle)\mathbf{1}(\Tilde{\Dc}^t_{-0})\right]}{\mathcal{P}(\Tilde{\Dc}^t_{-0})}\\
        % =&\frac{\mathrm{E}\left[\exp (\Xi^t_0 + 2\gamma\lambda \langle\xb,\big(\frac{\sum_{k=1}^m \wb^{k,t}_{-i}\Tilde{\xb}^{k,t\top}_{-i}}{m}\big)\tilde{\Hb}^{t,avg}_{1,B-1}\yb\rangle +\lambda\langle\xb,\Delta^t_{-1}\yb\rangle)\mathbbm{1}(\Tilde{\Dc}^t_{-0})\right]}{\mathcal{P}(\Tilde{\Dc}^t_{-0})}\\
        \leq&\frac{\mathrm{E}\left[\exp \left(\Xi^t_0 + 
        \varkappa^t_0+\lambda\langle\xb,\Delta^t_{-1}\yb\rangle\right)\mathbf{1}(\Tilde{\Dc}^t_{-0})\right]}{\mathcal{P}(\Tilde{\Dc}^t_{-0})}\\
        \leq &\frac{\mathrm{E}\left[\exp (\Xi^t_1 +\lambda\langle\xb,\Delta^t_{-1}\yb\rangle)\mathbf{1}(\Tilde{\Dc}^t_{-0})\right]}{\mathcal{P}(\Tilde{\Dc}^t_{-0})}\\
        \leq &\frac{\mathrm{E}\left[\exp (\Xi^t_1 +\lambda\langle\xb,\Delta^t_{-1}\yb\rangle)\mathbf{1}(\Tilde{\Dc}^t_{-1})\right]}{\mathcal{P}(\Tilde{\Dc}^t_{-0})},
    \end{split}
    \end{equation}}
    where the first inequality is due to the sub-Gaussianity of $\langle\xb,\wb^{k,t}_{-0}\rangle$ and the fact that $\forall k\in[m]$, $\wb^{k,t}_{-0}$ is independent of $\Xi^t_0$, $\tilde{\Hb}^{t,avg}_{1,B-1}$, $\Delta^t_{-1}$ and $\Tilde{\Dc}^t_{-0}$, the second inequality comes from \eqref{eq1: Bounding the exponential moment of the H product}, and the last inequality is based on $\Tilde{\Dc}^t_{-0}\subseteq\Tilde{\Dc}^t_{-1}$. Then, by expanding the recursive relationship in \eqref{eq2: Bounding the exponential moment of the H product}, the result is proved.
\end{proof}

\begin{corollary}\label{Corollary 1}
    From Lemma \ref{L:Bounding the exponential moment of the H product}, we have
    \begin{equation*}
    \begin{split}
        \mathrm{E}\big[\exp(\lambda \xb^{\top}\Delta^t \yb)|\Tilde{\Dc}^t_{-0}\big]
        \leq \frac{\exp (\frac{\gamma}{m}\lambda^2 C_{\mu}\xb^{\top}\Sigma\xb\norm{\yb}^2)}{\mathcal{P}(\Tilde{\Dc}^t_{-0})}.
    \end{split}
    \end{equation*}
\end{corollary}

\begin{lemma}\label{L:Contraction property of the H product}
    Suppose that $\gamma BR<\frac{1}{6}$. Then, under $\Tilde{D}^t_{-0}\:\forall t$ the following relationships hold: {\small
    \begin{equation*}
    \begin{split}
        &\Ib - 4\gamma\left(1+\frac{2\gamma BR}{1-4\gamma BR}\right)\sum_{i=0}^{B-1}\sum_{k=1}^m\Tilde{\xb}^{k,t}_{-i}\Tilde{\xb}^{k,t\top}_{-i} \preceq\tilde{\Hb}^{t,avg \top}_{0,B-1}\tilde{\Hb}^{t,avg}_{0,B-1}\\
        &\tilde{\Hb}^{t,avg \top}_{0,B-1}\tilde{\Hb}^{t,avg}_{0,B-1}\preceq\Ib - 4\gamma\left(1-\frac{2\gamma BR}{1-4\gamma BR}\right)\sum_{i=0}^{B-1}\sum_{k=1}^m\Tilde{\xb}^{k,t}_{-i}\Tilde{\xb}^{k,t\top}_{-i}.
    \end{split}
    \end{equation*}}
\end{lemma}
\begin{proof}
    The result can be derived by following the proof of Lemma 28 in \cite{kowshik2021streaming}.
\end{proof}

\begin{lemma}\label{L: Probability bound 1 of the H product}
    Suppose that $\gamma RB< \frac{1}{8}$, Assumption \ref{A: noise sub-gaussian} holds, and there exists a constant $c_0>0$ such that $c_0 < \frac{1}{128 C_{\mu}^2}$ and $1-c_0\leq \mathcal{P}(\Tilde{\Dc}^t_{-0})$. Then, for any $\xb\in \mathrm{R}^d$ we get
    \begin{equation*}
        \mathcal{P}\left(\norm{\tilde{\Hb}^{t,avg}_{0,B-1}\xb}^2\geq \norm{\xb}^2 - B\gamma \xb^{\top}\Gb\xb\middle|\Tilde{\Dc}^t_{-0}\right)\leq 1-p_0, 
    \end{equation*}
    where $p_0 := \frac{1}{1-c_0}(\frac{1}{128C_{\mu}^2}-c_0)$.
\end{lemma}
\begin{proof}
    The result can be derived using Lemma \ref{L:Contraction property of the H product} and the proof of Lemma 30 in \cite{kowshik2021streaming}.
\end{proof}

\begin{lemma}\label{L: Probability bound 2 of the H product}
    Suppose that the conditions in Lemma \ref{L: Probability bound 1 of the H product} hold and there exists a constant $c_1>0$, such that $c_1 < e(1-p_0)<(1-c_1)$ and $b-a > c_2:=\frac{e(1-p_0)-c_1}{(1-c_1) - e(1-p_0)}$. Then, conditioned on $\Tilde{\Dc}^{a,b}$, we have
    \begin{enumerate}
        \item $\norm{\prod_{s=a}^b \tilde{\Hb}^{s,avg}_{0,B-1}}\leq 1$ almost surely.
        \item $\mathcal{P}\left(\norm{\prod_{s=a}^b \tilde{\Hb}^{s,avg}_{0,B-1}}\geq 2(1-\gamma B \sigma_{\min}(\Gb))^{\frac{c_1}{2}(b-a+1)}\right) \\ \leq\exp (-c_3(b-a+1) + c_4d)$, where $c_3, c_4$ are constants, and $c_3$ depends on $C_{\mu}$.
    \end{enumerate}
\end{lemma}
\begin{proof}
    The result can be derived by following the proof of Lemma 31 in \cite{kowshik2021streaming}.
\end{proof}
Let us now define $F_{a,N-1}:= \sum_{r=a}^{N-1}\prod_{s=a+1}^r \tilde{\Hb}^{s,avg \top}_{0,B-1}$.
\begin{lemma}\label{L: Probability bound of the F term}
    Suppose that $c_1 \gamma B\sigma_{\min}(\Gb) < \frac{1}{2}$ and the conditions in Lemma \ref{L: Probability bound 2 of the H product} hold. Then, for any $\delta\in (0,1)$, we have
    \begin{equation*}
        \mathcal{P}\left(\norm{F_{a,N-1}}\geq c_5\big(\zeta + \frac{1}{\gamma B \sigma_{\min}(\Gb)}\big)\middle|\Tilde{\Dc}^{a,N-1}\right)\leq \delta,
    \end{equation*}
    where $\zeta := \max\{\frac{c_4d}{c_3} + \log(\frac{N}{\delta})/c_3 - 1, c_2\}$, and $c_5$ is a constant depending on $C_{\mu}$.
\end{lemma}

Consider the tail-averaged estimate:
\begin{equation*}
\begin{split}
    \hat{\Tilde{\Ab}}^{var,avg}_{0,N-1} &= 
    \frac{1}{N}\sum_{t=0}^{N-1}(\Tilde{\Ab}^{t,avg}_{B,var})= \frac{1}{N}\sum_{t=0}^{N-1} \Delta^t F_{t,N-1}.
\end{split}
\end{equation*}
Define the following events:
\begin{equation}\label{eq: Definition of event M}
\begin{split}
    \Tilde{\Mc}^t &:=  \left\{\norm{F_{t,N-1}} \leq c_5(\zeta + \frac{1}{\gamma B \sigma_{\min}}(\Gb))\right\},\\
    \Tilde{\Mc}^{t,N-1} &:= \cap_{s=t}^{N-1}\Tilde{\Mc}^s.
\end{split}
\end{equation}

\begin{theorem}\label{T: Bound of the tail average of the variance term}
    Suppose that the conditions in Lemma \ref{L:Bounding the exponential moment of the H product}, Lemma \ref{L:Contraction property of the H product} and Lemma \ref{L: Probability bound of the F term} hold. Assume $\mathcal{P}(\Tilde{\Mc}^{0,N-1}\cap \Tilde{\Dc}^{0,N-1}) \geq \frac{1}{2}$ and define $\alpha:= c_5(\zeta + \frac{1}{\gamma B\sigma_{\min}(\Gb)})$. Then we have 
    \begin{equation*}
    \begin{split}
        &\mathcal{P}\left(\norm{\hat{\Tilde{\Ab}}^{var,avg}_{0,N-1}}\geq \beta'\middle| \Tilde{\Mc}^{0,N-1}\cap \Tilde{\Dc}^{0,N-1}\right)\\
        \leq &\exp (c_6 d - \frac{\beta'^2 Nm}{8\gamma C_{\mu}\sigma_{\max}(\Sigma)(1+2\alpha)}),
    \end{split}
    \end{equation*}
    where $c_6$ is a problem-independent constant.
\end{theorem}

\begin{proof}
    The result can be derived by following the  proof of Theorem 33 in \cite{kowshik2021streaming}.
    % Note that $\Delta^t$ is independent with $F_{t_1, N-1}$ and $\Tilde{\Dc}^{t_2,N-1}$, where $t_1>=t, t_2>t$. Define $\Gamma_{t,N-1}:=\frac{1}{N}\sum_{s=t}^{N-1} \Delta^s F_{s,N-1}$. Then for any $\lambda >0$, $\xb,\yb \in \mathrm{R}^d$ such that  $\norm{\xb} = \norm{\yb}=1$, we have
    % \begin{equation}
    % \begin{split}
    %     &\mathrm{E}\left[\exp (\lambda\xb^{\top}(\hat{\Tilde{\Ab}}^{var,avg}_{0.N-1})\yb)\middle|\Tilde{\Mc}^{0,N-1}\cap \Tilde{\Dc}^{0,N-1}\right]\\
    %     =&\frac{\mathrm{E}\left[\exp (\lambda\xb^{\top}(\hat{\Tilde{\Ab}}^{var,avg}_{0.N-1})\yb)\mathbbm{1}(\Tilde{\Mc}^{0,N-1}\cap \Tilde{\Dc}^{0,N-1})\right]}{\mathcal{P}(\Tilde{\Mc}^{0,N-1}\cap \Tilde{\Dc}^{0,N-1})}\\
    %     =&\frac{\mathrm{E}\left[\exp (\frac{\lambda}{N}\xb^{\top}\Delta^0F_{0,N-1}\yb + \lambda \xb^{\top}\Gamma_{1,N-1}\yb)\mathbbm{1}(\Tilde{\Mc}^{0,N-1}\cap \Tilde{\Dc}^{1,N-1})\mathbbm{1}(\Tilde{\Dc}^0_{-0})\right]}{\mathcal{P}(\Tilde{\Mc}^{0,N-1}\cap \Tilde{\Dc}^{0,N-1})}\\
    %     \leq&\frac{\mathrm{E}\left[\exp (2\frac{\gamma}{m}\frac{\lambda^2 C_{\mu}}{N^2}\xb^{\top}\Sigma\xb\norm{F_{0,N-1}\yb}^2 + \lambda \xb^{\top}\Gamma_{1,N-1}\yb)\mathbbm{1}(\Tilde{\Mc}^{0,N-1}\cap \Tilde{\Dc}^{1,N-1})\right]}{\mathcal{P}(\Tilde{\Mc}^{0,N-1}\cap \Tilde{\Dc}^{0,N-1})},
    % \end{split}
    % \end{equation}
    % where the inequality comes from applying Corollary \ref{Corollary 1} with the fact that $\Delta^0$ and $\Tilde{\Dc}^0_{-0}$ are both independent with $F_{0,N-1}$, $\Gamma_{1,N-1}$, $\Tilde{\Mc}^{0,N-1}$ and $\Tilde{\Dc}^{1,N-1}$.
\end{proof}

\begin{lemma}\label{L: Coupled bound}
    Suppose that $2\gamma R<1$ and the event of $\hat{\Dc}^{0,N-1}$ holds. Then, for $0\leq t\leq N-1$ and $0\leq i\leq B-1$, we have 
    \begin{equation*}
        \norm{\Ab^{t,avg}_{i} - \Tilde{\Ab}^{t,avg}_{i}} \leq (16\gamma^2 R^2T^2 + 8\gamma RT)\norm{\Ab^u}.
    \end{equation*}
\end{lemma}
\begin{proof}
    Consider the update rule of $\Ab^{t,avg}_i,\:0\leq i\leq B-1 $:
    \begin{equation}\label{eq1: Coupled bound}
    \begin{split}
        \Ab^{t,avg}_{i+1} &= \Ab^{t,avg}_{i} - \frac{2\gamma}{m}\sum_{k=1}^m (\Ab^{t,avg}_{i}\xb^{k,t}_{-i} - \xb^{k,t}_{-(i-1)})\xb^{k,t\top}_{-i}\\
        &= \Ab^{t,avg}_{i} - \frac{2\gamma}{m}\sum_{k=1}^m (\Ab^{t,avg}_{i}\Tilde{\xb}^{k,t}_{-i} - \Tilde{\xb}^{k,t}_{-(i-1)})\Tilde{\xb}^{k,t\top}_{-i}\\
        &+ \frac{2\gamma}{m}\sum_{k=1}^m \Ab^{t,avg}_{i}\left(\Tilde{\xb}^{k,t}_{-i}\Tilde{\xb}^{k,t\top}_{-i} - \xb^{k,t}_{-i}\xb^{k,t\top}_{-i}\right)\\
        &+ \frac{2\gamma}{m}\sum_{k=1}^m \left(\xb^{k,t}_{-(i-1)}\xb^{k,t\top}_{-i} - \Tilde{\xb}^{k,t}_{-(i-1)}\Tilde{\xb}^{k,t\top}_{-i}\right).
    \end{split}
    \end{equation}
    Combining above with \eqref{eq:tilde}, we obtain
    \begin{equation}\label{eq2: Coupled bound}
    \begin{split}
        \Ab^{t,avg}_{i+1} - \Tilde{\Ab}^{t,avg}_{i+1} &= (\Ab^{t,avg}_{i} - \Tilde{\Ab}^{t,avg}_{i})\tilde{\Pb}^{t,avg}_{-i}\\
        &+ \frac{2\gamma}{m}\sum_{k=1}^m \Ab^{t,avg}_{i}\left(\Tilde{\xb}^{k,t}_{-i}\Tilde{\xb}^{k,t\top}_{-i} - \xb^{k,t}_{-i}\xb^{k,t\top}_{-i}\right)\\
        &+ \frac{2\gamma}{m}\sum_{k=1}^m \left(\xb^{k,t}_{-(i-1)}\xb^{k,t\top}_{-i} - \Tilde{\xb}^{k,t}_{-(i-1)}\Tilde{\xb}^{k,t\top}_{-i}\right).
    \end{split}
    \end{equation}
    For the term $\Ab^{t,avg}_{i}$, based on the update rule, the choice of $\gamma$ and the condition of $\hat{\Dc}^{0,N-1}$, we have
    \begin{equation*}
    \begin{split}
        \norm{\Ab^{t,avg}_{i+1}} &= \norm{\Ab^{t,avg}_{i}\Pb^{t,avg}_{-i} + \frac{2\gamma}{m}\sum_{k=1}^m \xb^{k,t}_{-(i-1)}\xb^{k,t\top}_{-i}}\\
        &\leq \norm{\Ab^{t,avg}_{i}}\norm{\Pb^{t,avg}_{-i}} + \frac{2\gamma}{m}\sum_{k=1}^m\norm{\xb^{k,t}_{-(i-1)}\xb^{k,t\top}_{-i}}\\
        &\leq \norm{\Ab^{t,avg}_{i}} + 2\gamma R,
    \end{split}
    \end{equation*}
    from which we conclude that for $0\leq t\leq N-1$ and $0\leq i\leq B-1$, 
    \begin{equation}\label{eq3: Coupled bound}
        \norm{\Ab^{t,avg}_{i}} \leq 2\gamma R T.    
    \end{equation}
    For the term $\left(\Tilde{\xb}^{k,t}_{-i}\Tilde{\xb}^{k,t\top}_{-i} - \xb^{k,t}_{-i}\xb^{k,t\top}_{-i}\right)$, it can be shown that
    \begin{equation}\label{eq4: Coupled bound}
    \begin{split}
        &\norm{\Tilde{\xb}^{k,t}_{-i}\Tilde{\xb}^{k,t\top}_{-i} - \xb^{k,t}_{-i}\xb^{k,t\top}_{-i}}\\
        =&\norm{\Tilde{\xb}^{k,t}_{-i}\Tilde{\xb}^{k,t\top}_{-i} - \xb^{k,t}_{-i}\Tilde{\xb}^{k,t\top}_{-i} + \xb^{k,t}_{-i}\Tilde{\xb}^{k,t\top}_{-i} - \xb^{k,t}_{-i}\xb^{k,t\top}_{-i}}\\
        \leq&2\sqrt{R}\norm{\xb^{k,t}_{-i} - \Tilde{\xb}^{k,t}_{-i}}=2\sqrt{R}\norm{\Ab^{(S-1)-i}(\xb^{k,t}_{0} - \Tilde{\xb}^{k,t}_{0})}\\
        \leq&4R\norm{\Ab^u}, 
    \end{split}
    \end{equation}
    where the first inequality is due to the condition of $\hat{\Dc}^{0,N-1}$ and the last inequality comes from the gap between buffers. Similarly, we can show $\norm{\xb^{k,t}_{-(i-1)}\xb^{k,t\top}_{-i} - \Tilde{\xb}^{k,t}_{-(i-1)}\Tilde{\xb}^{k,t\top}_{-i}}\leq 4R\norm{\Ab^u}$.
    
    Substituting \eqref{eq3: Coupled bound} and \eqref{eq4: Coupled bound} into \eqref{eq2: Coupled bound}, we get
    \begin{equation*}
    \begin{split}
        &\norm{\Ab^{t,avg}_{i+1} - \Tilde{\Ab}^{t,avg}_{i+1}}\\
        \leq &\norm{\Ab^{t,avg}_{i} - \Tilde{\Ab}^{t,avg}_{i}}\norm{\tilde{\Pb}^{t,avg}_{-i}} + (16\gamma^2 R^2T + 8\gamma R)\norm{\Ab^u}\\
        \leq &\norm{\Ab^{t,avg}_{i} - \Tilde{\Ab}^{t,avg}_{i}} + (16\gamma^2 R^2T + 8\gamma R)\norm{\Ab^u},
    \end{split}
    \end{equation*}
    from which we conclude that for $0\leq t\leq N-1$ and $0\leq i\leq B-1$, 
    \begin{equation*}
        \norm{\Ab^{t,avg}_{i} - \Tilde{\Ab}^{t,avg}_{i}} \leq (16\gamma^2 R^2T^2 + 8\gamma RT)\norm{\Ab^u}.    
    \end{equation*}
    
\end{proof}

\subsection{Difference between $\Ab^{t}_{i,k}$ and $\Ab^{t,avg}_{i}$}

\begin{lemma}\label{L: Bound on the difference between estimates from different networks}
    Suppose that $2\gamma R<1$ and the event of $\Dc^{0,N-1}$ holds. Then, for $0\leq t\leq N-1$ and $0\leq i\leq B-1$, we have for any agent $k$
    \begin{equation*}
        \norm{\Ab^{t}_{i,k} - \Ab^{t,avg}_{i}} \leq 2\sqrt{m}\beta^{\tau}\gamma RT^2.
    \end{equation*}
\end{lemma}

\begin{proof}
First, we note that using the recursive relationship \eqref{eq3: Coupled bound}, it is easy to verify that under the condition of $\Dc^{0,N-1}$, both $\norm{\Ab^{*t}_{i,k}}$ and $\norm{\Ab^{t,avg}_{i}}$ are upper-bounded by $2(S*t+i)\gamma R$. 
Let us define the following notations:
\begin{equation*}
\begin{split}
    \Bar{\Ab}^t_i &:=\frac{1}{m}\sum_{k=1}^m\Ab^{*t}_{i,k},\\
     \eb^{t}_{i,k} &:= \Ab^{t}_{i,k} - \Ab^{t,avg}_{i},\;\text{and }\norm{\eb^{t}_{i,k}} \leq E^{t}_{i}\;\forall k.
\end{split}
\end{equation*}
With Algorithm \ref{alg:DSGD-RER}, we have the following update expressions:
\begin{equation}\label{eq1: update of the avg net}
    \Ab^{t,avg}_{i+1} = \Ab^{t,avg}_{i} - \frac{2\gamma}{m}\sum_{k=1}^m (\Ab^{t,avg}_{i}\xb^{k,t}_{-i} - \xb^{k,t}_{-(i-1)})\xb^{k,t\top}_{-i},
\end{equation}
and 
\begin{equation}\label{eq2: update of the given net}
    \Ab^{*t}_{i+1,k} = \Ab^{t}_{i,k} - 2\gamma(\Ab^{t}_{i,k}\xb^{k,t}_{-i} - \xb^{k,t}_{-(i-1)}){(\xb^{k,t}_{-i})}^\top.
\end{equation}
Substituting $(\Ab^{t,avg}_{i} + \eb^{t}_{i,k})$ for $\Ab^{t}_{i,k}$ in \eqref{eq2: update of the given net}, with the definition of $\Bar{\Ab}^t_{i+1}$, we get
\begin{equation}\label{eq3: Diff between the mean and the estiamte of the avg net}
\begin{split}
    \norm{\Bar{\Ab}^t_{i+1} - \Ab^{t,avg}_{i+1}} &=  \norm{\frac{1}{m}\sum_{k=1}^m \eb^t_{i,k}(\Ib - 2\gamma \xb^{k,t}_{-i}\xb^{k,t\top}_{-i})}\\
    &\leq E^t_i.
\end{split}
\end{equation}
On the other hand, through step 9 in Algorithm \ref{alg:DSGD-RER}, it can be shown that $\forall k$,
\begin{equation}\label{eq4: Diff between the mean and the estiamte of each agent}
\begin{split}
    &\norm{\Ab^t_{i+1,k} - \Bar{\Ab}^t_{i+1}}\\
    = &\norm{\sum_{j}[\Pb^{\tau}]_{jk}\Ab^{*t}_{i+1,j} - \frac{1}{m}\sum_{j}\Ab^{*t}_{i+1,j}}\\
    \leq &\sum_{j}\left|[\Pb^{\tau}]_{jk}-\frac{1}{m}\right|\norm{\Ab^{*t}_{i+1,j}}\\
    \leq &\sqrt{m}2(S*t+i+1)\beta^{\tau}\gamma R.
\end{split}    
\end{equation}
With \eqref{eq3: Diff between the mean and the estiamte of the avg net} and \eqref{eq4: Diff between the mean and the estiamte of each agent}, the relationship between $E^{t}_{i+1}$ and $E^{t}_{i}$ is as follows:
\begin{equation}\label{eq5: Recursive update of error}
\begin{split}
    &\norm{\Ab^t_{i+1,k} - \Ab^{t,avg}_{i+1}}\\
    \leq &\norm{\Ab^t_{i+1,k} - \Bar{\Ab}^t_{i+1}} + \norm{\Bar{\Ab}^t_{i+1} - \Ab^{t,avg}_{i+1}}\\
    \leq &2\sqrt{m}(S*t+i+1)\beta^{\tau}\gamma R + E^t_i := E^t_{i+1}.
\end{split}
\end{equation}
Notice that with the zero initialization of all agents estimates, it can be shown that $E^0_1=2\sqrt{m}\beta^{\tau}\gamma R$. Expanding \eqref{eq5: Recursive update of error} recursively, we get
\begin{equation}\label{eq6: Diff between agent's estimate and the avg net's estimate}
\begin{split}
    E^t_{i}  %&2\sqrt{m}\beta^{\tau}\gamma R\frac{(S*t+i+1)(S*t+i)}{2}\\
    \leq &2\sqrt{m}\beta^{\tau}\gamma RT^2.
\end{split}
\end{equation}

\end{proof}

\end{document}

%% file: header.tex
\newcommand{\eb}{\mathbf{e}}

\newcommand{\wb}{\mathbf{w}}
\newcommand{\xb}{\mathbf{x}}
\newcommand{\yb}{\mathbf{y}}

\newcommand{\Ab}{\mathbf{A}}
\newcommand{\Bb}{\mathbf{B}}

\newcommand{\Gb}{\mathbf{G}}
\newcommand{\Hb}{\mathbf{H}}
\newcommand{\Ib}{\mathbf{I}}

\newcommand{\Pb}{\mathbf{P}}

\newcommand{\Ub}{\mathbf{U}}

\newcommand{\Dc}{\mathcal{D}}

\newcommand{\Mc}{\mathcal{M}}
\newcommand{\Nc}{\mathcal{N}}

\newcommand{\argmin}{\text{argmin}}

\newcommand{\norm}[1]{\left\lVert#1\right\rVert}

\newtheorem{theorem}{Theorem}

\newtheorem{assumption}{Assumption}

\newtheorem{corollary}[theorem]{Corollary}

\newtheorem{lemma}[theorem]{Lemma}

\newtheorem{remark}{Remark}